\documentclass[11pt]{article}

\usepackage[margin=1in]{geometry}
\usepackage{amsmath,amssymb}
\usepackage{amsthm}

\usepackage[noblocks]{authblk} % 'noblocks' option: always use the footnote mode 
\allowdisplaybreaks[3]

\usepackage[T1]{fontenc}
\usepackage[utf8]{inputenc}
\usepackage{graphicx}
\usepackage[font=small,margin=.05\textwidth]{caption} % Reduce the font size of figure captions.
\usepackage[font=scriptsize]{subcaption}
\usepackage[english]{babel}
\usepackage{verbatim}

\usepackage[commandnameprefix=ifneeded]{changes}
\setdeletedmarkup{{\color{gray}[cut: #1]}}

\theoremstyle{plain}
\newtheorem{thm}{\protect\theoremname}
\theoremstyle{plain}
\newtheorem{prop}[thm]{\protect\propositionname}
\theoremstyle{definition}

\theoremstyle{definition}
\newtheorem{defn}[thm]{\protect\definitionname}
\theoremstyle{plain}

\theoremstyle{plain}
\newtheorem{lem}[thm]{\protect\lemmaname}
\theoremstyle{plain}
\newtheorem*{lem*}{\protect\lemmaname}
\theoremstyle{plain}
\newtheorem*{prop*}{\protect\propositionname}
\theoremstyle{remark}
\newtheorem{rem}[thm]{\protect\remarkname}

\providecommand{\corollaryname}{Corollary}
\providecommand{\definitionname}{Definition}
\providecommand{\examplename}{Example}
\providecommand{\lemmaname}{Lemma}
\providecommand{\propositionname}{Proposition}
\providecommand{\theoremname}{Theorem}
\providecommand{\remarkname}{Remark}

\usepackage[authoryear,sort]{natbib}
\usepackage[unicode=true,
			bookmarks=true,
			bookmarksnumbered=false,
			bookmarksopen=false,
			breaklinks=false,
			pdfborder={0 0 1},
			backref=false,
			hypertexnames=false, % hypertexnames=false is needed for autonum
			colorlinks=false]{hyperref}

\usepackage{mathtools}
\usepackage{autonum}
\makeatletter
\newcommand{\restore@Environment}[1]{%
  \AtBeginDocument{%
  	\csletcs{#1*}{#1}%
  	\csletcs{end#1*}{end#1}%
  }%
}
\forcsvlist\restore@Environment{alignat,equation,gather,multline,flalign,align} 
\makeatother

\usepackage{algpseudocode}
\usepackage{algorithm}

\usepackage{acro}
\acsetup{
	single = true % If an acronym is used only once in the paper, spell it out and don't make an acronym
}

\usepackage{xcolor}
\definecolor{darkgreen}{rgb}{0,0.5,0}

\newcommand{\abs}[1]{\left\vert#1\right\vert}
\newcommand{\norm}[1]{\left\Vert#1\right\Vert}
\newcommand{\cnorm}[1]{\|#1\|}							% Norm in text style

\newcommand{\reals}{\mathbb{R}}

\newcommand{\deq}{\triangleq}							% Defined as

\newcommand{\E}{\mathbb{E}}								% Expectation
\newcommand{\PP}{\mathbb{P}}

\newcommand{\klmax}[2]{D_\mathrm{KL}^{\max}(#1 \mathrel{\Vert} #2)}

\newcommand{\eos}{\textrm{EOS}}  % end of sequence token
\newcommand{\piapp}{\hat{\pi}}  % approximate token-level policy
\newcommand{\Japp}{\hat{J}}  % approximate objective function

\newcommand{\Cov}{\operatorname{Cov}}
\newcommand{\Var}{\operatorname{Var}}
\newcommand{\tr}{\operatorname{tr}}
\newcommand{\score}{s_{\theta}}
\newcommand{\wscore}{\tilde{s}}

\newcommand{\cA}{\mathcal{A}}

\newcommand{\cD}{\mathcal{D}}

\newcommand{\cV}{\mathcal{V}}

\newcommand{\cX}{\mathcal{X}}
\newcommand{\cY}{\mathcal{Y}}

\DeclareAcronym{rl}{
	short = RL ,
	long  = reinforcement learning
}

\DeclareAcronym{mdp}{
	short      = MDP ,
	long       = Markov decision process ,
	long-plural-form = Markov decision processes ,
	short-indefinite = an
}

\DeclareAcronym{rloo}{
	short      = RLOO ,
	long       = REINFORCE leave-one-out ,
	short-indefinite = an
}

\DeclareAcronym{trpo}{
	short = TRPO ,
	long  = trust-region policy optimization
}

\DeclareAcronym{ppo}{
	short = PPO ,
	long  = proximal policy optimization
}

\DeclareAcronym{grpo}{
	short = GRPO ,
	long  = group relative policy optimization
}

\begin{document}

\title{Baselines and Score Centering in Policy Gradient Estimation}

\author[1]{Shuo Han\thanks{Email: \texttt{hanshuo@uic.edu}}}
\author[2]{Ziqing Hu\thanks{Email: \texttt{ziqing@perplexity.ai}}}
\affil[1]{University of Illinois Chicago}
\affil[2]{Perplexity AI}

\date{October 15, 2026}

\maketitle

\begin{abstract}
	Reinforcement learning has become a widely used approach to post-training large language models. In a typical implementation, learning is off-policy: The policy generating the responses differs from the policy being trained because of training--inference mismatch or policy staleness in asynchronous training. This discrepancy is commonly addressed by importance sampling, whose ratio is usually regularized in practice through truncation, masking, or clipping. Under regularized importance sampling, the importance-weighted policy gradient is no longer invariant to constant shifts in the reward. This invariance is recovered by using the covariance between the reward and the importance-weighted score as the estimand, which coincides with the policy gradient when the ratio is not regularized. This note develops this covariance perspective for both sequence-level and token-level importance sampling. The covariance perspective reveals the roles of the group-mean baseline and score centering in policy gradient estimation. Both perform centering in covariance estimation, but they differ in the mean being subtracted: The group-mean baseline uses a sample mean, whereas score centering uses a population mean. The two lead to distinct unbiased covariance estimators with different variances. Whether both forms of centering should be used simultaneously depends on whether importance sampling is applied at the sequence level or at the token level. With sequence-level importance sampling, score centering adds no benefit once the group-mean baseline is used; with token-level importance sampling, using the group-mean baseline without score centering generally introduces bias when the importance ratio is regularized.
\end{abstract}

% Reset acronyms so the main text treats them as first-use again
\acresetall

\section{Background}

\subsection{Post-training of language models}

In post-training of language models, the goal is to make the language model generate a desired response for a given prompt. Denote by $x\in\cX$ the prompt and $y\in\cY$ the response. The language model can be viewed as a policy $\pi$, where $\pi(y\mid x)$ denotes the probability of generating $y$ given the prompt $x$. The response $y$ consists of a sequence of tokens $\{y_{t}\}_{t\geq0}$, each of which is randomly generated based on past tokens and the prompt. The generation of one token is known as a \emph{decoding} step. We use $\pi$ both for the distribution over complete responses and, as in $\pi(y_{t}\mid x,y_{<t})$, for the conditional distribution of a single token given the prompt $x$ and the history $y_{<t}\deq(y_{0},\dots,y_{t-1})$ when it is clear from the context. Thus, for any response $y$, it holds that 
\[
\pi(y\mid x)=\prod_{t=0}^{|y|-1}\pi(y_{t}\mid x,y_{<t}).
\]
The quality of $y$ for a given $x$ is measured by a reward function $R\colon\mathcal{X}\times\mathcal{Y}\to\reals$ as $R(x,y)$. Let $\cD$ be the distribution of prompts. The goal is to find $\pi$ that maximizes 
\begin{equation}
	\E_{x\sim\mathcal{D},y\sim\pi(\cdot\mid x)}\left[R(x,y)\right].\label{eq:train_obj_seq_original}
\end{equation}
Throughout, we assume that generation terminates with probability one under every policy under consideration, so that $|y|<\infty$ almost surely and the expectation in~\eqref{eq:train_obj_seq_original} is well defined. In practice, this assumption is enforced by imposing a maximum response length.

\subsection{Token generation as \iacl*{mdp}\label{subsec:MDP}}

The token generation process can be modeled as an undiscounted indefinite-horizon \ac{mdp}. Let the state space $\mathcal{S}$ consist of sequences of tokens whose prefix is a valid prompt, and the action space $\mathcal{A}$ be the space of tokens. The idea is to use the state to store the prompt and all the tokens generated thus far, and to use the action to represent the next generated token:
\[
s_{t}\deq(x,y_{<t}),\qquad a_{t}\deq y_{t}.
\]
The initial state is chosen as $s_{0}=x$. At time $t$, the action $a_{t}$ is determined by the policy as $a_{t}\sim\pi(\cdot\mid s_{t})$. The transition dynamics are deterministic: $s_{t+1}=(x,y_{\leq t})$ is the concatenation of $s_{t}$ and $a_{t}$ unless $a_{t}=\eos$, the \emph{end-of-sequence} token. If $a_{t}=\eos$, the next state $s_{t+1}$ becomes an absorbing state associated with no reward. The original reward function $R$ induces a token-level reward $r\colon\mathcal{S}\times\mathcal{A}\to\reals$. During token generation, if $a_{t}=\eos$, then $r(s_{t},a_{t})=R(x,y_{\leq t})$, where $y_{\leq t}$ denotes the completed response including the terminal $\eos$ token, so that $|y|=t+1$; otherwise $r(s_{t},a_{t})=0$. In other words, the reward is sparse and assigned only at the last decoding step. Under the \ac{mdp} framework, the training objective in~\eqref{eq:train_obj_seq_original} can be equivalently rewritten as 
\[
\E_{x\sim\mathcal{D},y\sim\pi(\cdot\mid x)}\left[R(x,y)\right]=\E_{s_{0}\sim\mathcal{D}}^{\pi}\left[\sum_{t=0}^{\infty}r(s_{t},a_{t})\right].
\]
Define the state value, state-action value, and advantage functions of a policy $\pi$ by 
\begin{align*}
	V^{\pi}(s) & \deq\E^{\pi}\left[\left.\sum_{t=0}^{\infty}r(s_{t},a_{t})\,\right|\,s_{0}=s\right],\qquad Q^{\pi}(s,a)\deq\E^{\pi}\left[\left.\sum_{t=0}^{\infty}r(s_{t},a_{t})\,\right|\,s_{0}=s,a_{0}=a\right],\\
	A^{\pi}(s,a) & \deq Q^{\pi}(s,a)-V^{\pi}(s).
\end{align*}
In what follows, we assume that the policy is parametrized by $\theta\in\Theta$ for some set $\Theta$, and we denote the policy by $\pi_{\theta}$ to highlight the parameter dependence. The training objective in~\eqref{eq:train_obj_seq_original} thus becomes a function of $\theta$, defined by 
\begin{equation}
	J(\theta)\deq\E_{x\sim\mathcal{D},y\sim\pi_{\theta}(\cdot\mid x)}\left[R(x,y)\right].\label{eq:train_obj_seq}
\end{equation}

\section{Importance Sampling in Policy Optimization\label{sec:IS}}

The training objective in~\eqref{eq:train_obj_seq} involves an expectation over the prompt distribution $\cD$ and the distribution $\pi_{\theta}(\cdot\mid x)$ over possible responses. Because $\cD$ is unknown and the space of responses is prohibitively large to enumerate, it is intractable to compute the expectation exactly. Instead, the expectation is typically replaced by a sample-average approximation using sampled token sequences. A naive application of sample-average approximation to~\eqref{eq:train_obj_seq} requires these samples to be drawn from $\pi_{\theta}$.

However, modern post-training systems for reinforcement learning typically use separate engines with different design objectives. A \emph{training engine}, built on PyTorch FSDP~\citep{zhao_pytorch_2023} or Megatron-LM~\citep{shoeybi_megatron-lm_2019}, supports efficient gradient computation and stable updates to $\theta$, whereas an \emph{inference engine}, e.g., vLLM~\citep{kwon_efficient_2023} and SGLang~\citep{zheng_sglang_2024}, prioritizes generation throughput. Their distinct numerical implementations can yield different token probabilities even with identical policy parameters. This discrepancy is referred to as \emph{training--inference mismatch}~\citep{liu_when_2025}. For a given parameter $\theta$, let $\pi_{\theta}$ and $\piapp$ denote the policies implemented by the training and inference engines, respectively.

When samples are drawn from a different policy, importance sampling is a common technique for obtaining an unbiased estimate. It works by weighting the samples using an importance sampling ratio. Common formulations use two types of importance sampling ratios. One is sequence-level importance sampling, which uses 
\[
\frac{\pi_{\theta}(y\mid x)}{\piapp(y\mid x)}=\prod_{t=0}^{|y|-1}\frac{\pi_{\theta}(y_{t}\mid x,y_{<t})}{\piapp(y_{t}\mid x,y_{<t})}.
\]
The other is token-level importance sampling, which uses the individual ratios 
\[
\frac{\pi_{\theta}(y_{t}\mid x,y_{<t})}{\piapp(y_{t}\mid x,y_{<t})},\quad t=0,\dots,|y|-1
\]
rather than their product as a whole. As shown below, the two types of importance sampling ratios lead to different estimators, which are unbiased for different quantities. 

While sampling a response $y$ from $\pi_{\theta}$ must be done sequentially, scoring $\pi_{\theta}(y\mid x)$ when $y$ is given can be parallelized. Only one forward pass is required to compute all the $|y|$ quantities $\pi_{\theta}(y_{0}\mid x)$, $\pi_{\theta}(y_{1}\mid x,y_{0})$, \ldots, and this is typically much faster than sampling. Both formulations require the inference engine to return the token probabilities $\piapp(y_{t}\mid x,y_{<t})$ it used at decoding time.

\subsection{Methods using sequence-level importance sampling\label{subsec:sequence-level_IS}}

\subsubsection{REINFORCE}

The REINFORCE algorithm~\citep{williams_simple_1992} updates $\theta$ by performing gradient ascent on the training objective $J$ in~\eqref{eq:train_obj_seq}. It is well known \citep[see, e.g.,][]{sutton_reinforcement_2018} that the gradient is given by 
\begin{equation}
	\nabla J(\theta)=\E_{x\sim\mathcal{D},y\sim\pi_{\theta}(\cdot\mid x)}\left[R(x,y)\nabla_{\theta}\log\pi_{\theta}(y\mid x)\right].\label{eq:pg_seq}
\end{equation}
To use samples drawn from the inference policy $\piapp$, we rewrite~\eqref{eq:pg_seq} as 
\begin{equation}
	\nabla J(\theta)=\E_{x\sim\cD,y\sim\piapp(\cdot\mid x)}\left[\frac{\pi_{\theta}(y\mid x)}{\piapp(y\mid x)}R(x,y)\nabla_{\theta}\log\pi_{\theta}(y\mid x)\right],\label{eq:pg_seq_IS}
\end{equation}
assuming that $\piapp(y\mid x)>0$ whenever $\pi_{\theta}(y\mid x)>0$. For a given prompt $x$, draw samples $y_{1},\dots,y_{G}\sim\piapp(\cdot\mid x)$. From the samples, $\nabla J(\theta)$ can be approximated by a sample-average approximation as 
\begin{equation}
	\nabla J(\theta)\approx\frac{1}{G}\sum_{i=1}^{G}\frac{\pi_{\theta}(y_{i}\mid x)}{\piapp(y_{i}\mid x)}R(x,y_{i})\nabla_{\theta}\log\pi_{\theta}(y_{i}\mid x).\label{eq:pg_seq_saa}
\end{equation}
It can be shown that the right side of~\eqref{eq:pg_seq_saa} is an unbiased estimator of $\nabla J(\theta)$. 

\begin{rem}
	\label{rem:conditional}Strictly speaking, for a fixed prompt $x$, the estimator in~\eqref{eq:pg_seq_saa} is only unbiased for the inner expectation over $y\sim\piapp(\cdot\mid x)$ in~\eqref{eq:pg_seq_IS}. It is unbiased for $\nabla J(\theta)$ when the prompt $x$ is drawn from $\cD$. This is what happens in practice, where a batch of prompts is sampled and the per-prompt estimators are averaged. For brevity, we will assume that $x$ is drawn from $\cD$ in all subsequent discussions unless mentioned otherwise.
\end{rem}

\subsubsection{Policy gradient with a baseline\label{subsec:baseline}}

To reduce the variance of sample-average approximation, one important technique is introducing a baseline~\citep{greensmith_variance_2004} to the reward in $\nabla J(\theta)$. Let $b$ be a baseline, which may be random and may depend on the prompt $x$. It turns out that $\nabla J(\theta)$ remains unchanged after $b$ is introduced to $R(x,y)$, i.e., 
\[
\nabla J(\theta)=\E_{x\sim\cD,y\sim\piapp(\cdot\mid x)}\left[\frac{\pi_{\theta}(y\mid x)}{\piapp(y\mid x)}\left(R(x,y)-b\right)\nabla_{\theta}\log\pi_{\theta}(y\mid x)\right]
\]
provided that $b$ is conditionally independent of $y$ given $x$. This can be shown by recalling the fact 
\[
\E_{y\sim\piapp(\cdot\mid x)}\left[\frac{\pi_{\theta}(y\mid x)}{\piapp(y\mid x)}\nabla_{\theta}\log\pi_{\theta}(y\mid x)\right]=\E_{y\sim\pi_{\theta}(\cdot\mid x)}\left[\nabla_{\theta}\log\pi_{\theta}(y\mid x)\right]=0.
\]
If $b$ and $y$ are conditionally independent given $x$, the extra term contributed by the baseline becomes
\[
\E_{y\sim\piapp(\cdot\mid x)}\left[b\cdot\frac{\pi_{\theta}(y\mid x)}{\piapp(y\mid x)}\nabla_{\theta}\log\pi_{\theta}(y\mid x)\right]=b\cdot\E_{y\sim\piapp(\cdot\mid x)}\left[\frac{\pi_{\theta}(y\mid x)}{\piapp(y\mid x)}\nabla_{\theta}\log\pi_{\theta}(y\mid x)\right]=0.
\]
The use of a baseline can be extended to multiple samples. Let $y_{1},\dots,y_{G}\sim\piapp(\cdot\mid x)$ be drawn independently and let $b_{1},\dots,b_{G}$ be baselines. If $b_{i}$ is conditionally independent of $y_{i}$ given $x$ for $i=1,\dots,G$, it can be shown that the sample average 
\begin{equation}
	\frac{1}{G}\sum_{i=1}^{G}\left[\frac{\pi_{\theta}(y_{i}\mid x)}{\piapp(y_{i}\mid x)}\left(R(x,y_{i})-b_{i}\right)\nabla_{\theta}\log\pi_{\theta}(y_{i}\mid x)\right]\label{eq:rloo_seq}
\end{equation}
is an unbiased estimator of $\nabla J(\theta)$. Note that $b_{i}$ is allowed to depend on the other samples $\{y_{j}\}_{j\neq i}$ because those are independent of $y_{i}$ given $x$.

One popular choice of the baseline, as in the \ac{rloo} algorithm~\citep{kool_buy_2019}, is setting 
\begin{equation}
	b_{i}=\frac{1}{G-1}\sum_{j\neq i}R(x,y_{j}),\quad i=1,\dots,G,\label{eq:loo_baseline}
\end{equation}
which is the empirical mean reward of the group with the $i$th sample removed. Define 
\begin{equation}
	\bar{R}\deq\frac{1}{G}\sum_{j=1}^{G}R(x,y_{j}).\label{eq:group_mean}
\end{equation}
Using the fact 
\begin{equation}
	\frac{1}{G}\left(R(x,y_{i})-b_{i}\right)=\frac{1}{G-1}\left(R(x,y_{i})-\bar{R}\right),\label{eq:loo_identity}
\end{equation}
one can rewrite~\eqref{eq:rloo_seq} as 
\begin{equation}
	\frac{1}{G-1}\sum_{i=1}^{G}\left[\frac{\pi_{\theta}(y_{i}\mid x)}{\piapp(y_{i}\mid x)}\left(R(x,y_{i})-\bar{R}\right)\nabla_{\theta}\log\pi_{\theta}(y_{i}\mid x)\right].\label{eq:rloo_seq_mean_form}
\end{equation}

\begin{rem}
	\label{rem:drgrpo}Equation~\eqref{eq:loo_identity} shows that the leave-one-out baseline in~\eqref{eq:loo_baseline} is closely related to the group-mean baseline, defined by $b_{i}=\bar{R}$ for $i=1,\dots,G$. The group-mean baseline is used by the Dr.\ GRPO algorithm~\citep{liu_understanding_2025}, a variant of the \ac{grpo} algorithm~\citep{shao_deepseekmath_2024}, and leads to the following gradient estimator:
	\[
	\frac{1}{G}\sum_{i=1}^{G}\left[\frac{\pi_{\theta}(y_{i}\mid x)}{\piapp(y_{i}\mid x)}\left(R(x,y_{i})-\bar{R}\right)\nabla_{\theta}\log\pi_{\theta}(y_{i}\mid x)\right].
	\]
	Because $b_{i}=\bar{R}$ contains $R(x,y_{i})$, it is \emph{not} conditionally independent of $y_{i}$. However, the resulting gradient estimator only differs from the \ac{rloo} gradient estimator by a deterministic multiplicative factor $(G-1)/G$, which is absorbed into the learning rate in practice. In other words, Dr.\ GRPO is practically equivalent to the \ac{rloo} algorithm.
\end{rem}

\subsection{Methods using token-level importance sampling\label{subsec:token-level_IS}}

\subsubsection{\Acl*{trpo} and \acl*{ppo}}

Methods that use token-level importance sampling consider a surrogate objective function~\citep{kakade_approximately_2002} 
\[
\Japp(\theta)\deq\E_{s\sim d_{\piapp},a\sim\pi_{\theta}(\cdot\mid s)}\left[A^{\piapp}(s,a)\right],
\]
where 
\[
d_{\piapp}(s)=\sum_{t=0}^{\infty}\PP^{\piapp}\left(s_{t}=s\mid s_{0}\sim\cD\right)
\]
is the (unnormalized) state occupancy measure induced by $\piapp$ and the prompt distribution $\cD$. Because $d_{\piapp}$ is unnormalized, the slightly abused notation $s\sim d_{\piapp}$ denotes integration against this measure. We briefly recall why $\Japp$ is a useful surrogate function to optimize by examining the behavior of the training objective in the \emph{policy space}; namely, we will examine how $J(\theta)$ changes with the policy $\pi_{\theta}$ rather than $\theta$ itself. Define 
\[
V_{0}(\pi)\deq\E_{s_{0}\sim\cD}\left[V^{\pi}(s_{0})\right]=\E_{s_{0}\sim\cD}^{\pi}\left[\sum_{t=0}^{\infty}r(s_{t},a_{t})\right],
\]
which is the value of a given policy. The value of a parametrized policy $\pi_{\theta}$ recovers the training objective: $J(\theta)=V_{0}(\pi_{\theta})$. Denote by $\hat{V}_{0}(\cdot;\piapp)$ the linear approximation of $V_{0}$ at $\piapp$: 
\[
\hat{V}_{0}(\pi;\piapp)\deq V_{0}(\piapp)+DV_{0}(\piapp)\cdot\left(\pi-\piapp\right),
\]
where $DV_{0}(\piapp)$ denotes the derivative of $V_{0}$ at $\piapp$. It can be shown that
\begin{equation}
	\Japp(\theta)=\hat{V}_{0}(\pi_{\theta};\piapp)-V_{0}(\piapp)=DV_{0}(\piapp)\cdot(\pi_{\theta}-\piapp).\label{eq:Jhat_vs_V0}
\end{equation}
One can see from~\eqref{eq:Jhat_vs_V0} that $\Japp(\theta)$ is affine in $\pi_{\theta}$. Informally, $\Japp$ can be viewed as a linear approximation of $J$ in the policy space, up to an additive constant.

The relationship between $\Japp$ and $\hat{V}_{0}$ reveals why $\Japp$ is a useful surrogate objective to optimize. Because $\Japp(\theta)$ and $\hat{V}_{0}(\pi_{\theta};\piapp)$ differ only by a constant independent of $\theta$, maximizing $\Japp(\theta)$ is equivalent to maximizing $\hat{V}_{0}(\pi_{\theta};\piapp)$. If $\pi_{\theta}$ improves over $\piapp$ in $\hat{V}_{0}(\cdot;\piapp)$, i.e., $\hat{V}_{0}(\pi_{\theta};\piapp)>\hat{V}_{0}(\piapp;\piapp)=V_{0}(\piapp)$, this implies $DV_{0}(\piapp)\cdot(\pi_{\theta}-\piapp)>0$, indicating that $\pi_{\theta}-\piapp$ is an ascent direction in the policy space. When $\pi_{\theta}$ is sufficiently close to $\piapp$, it follows that $V_{0}(\pi_{\theta})>V_{0}(\piapp)$, i.e., $\pi_{\theta}$ improves over $\piapp$ in the true objective. In the constrained version of \ac{trpo}~\citep{schulman_trust_2015}, this condition is enforced by a trust-region constraint that restricts the divergence $\klmax{\pi_{\theta}}{\piapp}$.

To optimize $\Japp$ using a first-order method, one can obtain $\nabla\Japp(\theta)$ as 
\begin{equation}
	\nabla\Japp(\theta)=\E_{s\sim d_{\piapp},a\sim\pi_{\theta}(\cdot\mid s)}\left[A^{\piapp}(s,a)\nabla_{\theta}\log\pi_{\theta}(a\mid s)\right].\label{eq:trpo_grad}
\end{equation}
In comparison, it can be shown that $\nabla J(\theta)$ can be rewritten as
\[
\nabla J(\theta)=\E_{s\sim d_{\pi_{\theta}},a\sim\pi_{\theta}(\cdot\mid s)}\left[A^{\pi_{\theta}}(s,a)\nabla_{\theta}\log\pi_{\theta}(a\mid s)\right]
\]
under mild conditions. This shows that generally $\nabla\Japp(\theta)\neq\nabla J(\theta)$, i.e., $\hat{J}$ is not a linear approximation of $J$ in the \emph{parameter space}. To use samples from $\piapp$, we rewrite~\eqref{eq:trpo_grad} as 
\begin{align*}
	\nabla\Japp(\theta) & =\E_{s\sim d_{\piapp},a\sim\piapp(\cdot\mid s)}\left[\frac{\pi_{\theta}(a\mid s)}{\piapp(a\mid s)}A^{\piapp}(s,a)\nabla_{\theta}\log\pi_{\theta}(a\mid s)\right]\\
	& =\E_{s_{0}\sim\cD}^{\piapp}\left[\sum_{t=0}^{\infty}\frac{\pi_{\theta}(a_{t}\mid s_{t})}{\piapp(a_{t}\mid s_{t})}A^{\piapp}(s_{t},a_{t})\nabla_{\theta}\log\pi_{\theta}(a_{t}\mid s_{t})\right].
\end{align*}
The importance sampling ratio corrects the action distribution at each visited state, but the states are still distributed according to $d_{\piapp}$. Using the fact $s_{t}=(x,y_{<t})$ and $a_{t}=y_{t}$, we can further write 
\begin{equation}
	\nabla\Japp(\theta)=\E_{x\sim\cD,y\sim\piapp(\cdot\mid x)}\left[\sum_{t=0}^{|y|-1}\frac{\pi_{\theta}(y_{t}\mid x,y_{<t})}{\piapp(y_{t}\mid x,y_{<t})}A^{\piapp}((x,y_{<t}),y_{t})\nabla_{\theta}\log\pi_{\theta}(y_{t}\mid x,y_{<t})\right],\label{eq:trpo_grad_token}
\end{equation}
where the sum terminates at $t=|y|-1$ because the absorbing state reached after $\eos$ has no contribution. For a given prompt $x$, sample $y_{1},\dots,y_{G}\sim\piapp(\cdot\mid x)$, with which $\nabla\Japp(\theta)$ can be approximated as 
\[
\frac{1}{G}\sum_{i=1}^{G}\sum_{t=0}^{|y_{i}|-1}\frac{\pi_{\theta}(y_{i,t}\mid x,y_{i,<t})}{\piapp(y_{i,t}\mid x,y_{i,<t})}A^{\piapp}((x,y_{i,<t}),y_{i,t})\nabla_{\theta}\log\pi_{\theta}(y_{i,t}\mid x,y_{i,<t}).
\]

The surrogate objective $\Japp$ is also used in the \ac{ppo} algorithm~\citep{schulman_proximal_2017}. Compared with \ac{trpo}, \ac{ppo} replaces the trust-region constraint by a heuristic that involves clipping the (token-level) importance sampling ratio in $\Japp$ to a bounded interval:
\[
\E_{s\sim d_{\piapp},a\sim\piapp(\cdot\mid s)}\left[\min\left(\frac{\pi_{\theta}(a\mid s)}{\piapp(a\mid s)}A^{\piapp}(s,a),\ \mathrm{clip}\left(\frac{\pi_{\theta}(a\mid s)}{\piapp(a\mid s)},1-\epsilon,1+\epsilon\right)A^{\piapp}(s,a)\right)\right]
\]
for some $\epsilon>0$. The heuristic aims to discourage the updated policy from leaving the trust region by limiting the amount of improvement made by the updated policy. 

\subsubsection{Token-level GRPO}

In language model post-training, the value of $A^{\piapp}((x,y_{i,<t}),y_{i,t})$ is often difficult to evaluate or estimate, as it would require learning a value function over prefixes. To avoid evaluating $A^{\piapp}$, one can use the fact 
\begin{equation}
	\nabla\Japp(\theta)=\E_{x\sim\cD,y\sim\piapp(\cdot\mid x)}\left[R(x,y)\sum_{t=0}^{|y|-1}\frac{\pi_{\theta}(y_{t}\mid x,y_{<t})}{\piapp(y_{t}\mid x,y_{<t})}\nabla_{\theta}\log\pi_{\theta}(y_{t}\mid x,y_{<t})\right]\label{eq:trpo_grad_token_reward}
\end{equation}
to use $R$ in place of $A^{\piapp}$ (see, e.g., \citet{schulman_high-dimensional_2016}). A baseline can be incorporated similarly to what has been done in Section~\ref{subsec:baseline}. For any baseline $b$ that is conditionally independent of $y$ given $x$, it can be shown that 
\[
\nabla\Japp(\theta)=\E_{x\sim\cD,y\sim\piapp(\cdot\mid x)}\left[\left(R(x,y)-b\right)\sum_{t=0}^{|y|-1}\frac{\pi_{\theta}(y_{t}\mid x,y_{<t})}{\piapp(y_{t}\mid x,y_{<t})}\nabla_{\theta}\log\pi_{\theta}(y_{t}\mid x,y_{<t})\right].
\]
Suppose samples $y_{1},\dots,y_{G}\sim\piapp(\cdot\mid x)$ are available. Define $b_{i}$ as in~\eqref{eq:loo_baseline} for $i=1,\dots,G$, and $\bar{R}$ as in~\eqref{eq:group_mean}. It can be shown that 
\begin{equation}
	\frac{1}{G-1}\sum_{i=1}^{G}\left[\left(R(x,y_{i})-\bar{R}\right)\sum_{t=0}^{|y_{i}|-1}\frac{\pi_{\theta}(y_{i,t}\mid x,y_{i,<t})}{\piapp(y_{i,t}\mid x,y_{i,<t})}\nabla_{\theta}\log\pi_{\theta}(y_{i,t}\mid x,y_{i,<t})\right]\label{eq:token_grpo}
\end{equation}
is an unbiased estimator of $\nabla\Japp(\theta)$. 

\begin{rem}
	\label{rem:clipping}
	The estimator in~\eqref{eq:token_grpo} is the token-level counterpart of~\eqref{eq:rloo_seq_mean_form}. In practice, a clipping heuristic is sometimes applied to the token-level importance ratio in~\eqref{eq:token_grpo} to further reduce variance at the price of introducing bias, so that the gradient estimator becomes
	\[
	\frac{1}{G-1}\sum_{i=1}^{G}\left[\left(R(x,y_{i})-\bar{R}\right)\sum_{t=0}^{|y_{i}|-1}\mathrm{clip}\left(\frac{\pi_{\theta}(y_{i,t}\mid x,y_{i,<t})}{\piapp(y_{i,t}\mid x,y_{i,<t})},1-\epsilon,1+\epsilon\right)\nabla_{\theta}\log\pi_{\theta}(y_{i,t}\mid x,y_{i,<t})\right].
	\]
	The bias of the estimator will be discussed in Section~\ref{subsec:token}.
\end{rem}

\subsection{Comparison between sequence-level and token-level importance sampling\label{subsec:comparison}}

The results in Sections~\ref{subsec:sequence-level_IS} and~\ref{subsec:token-level_IS} show that the two importance sampling formulations yield unbiased gradient estimators for different objective functions: the sequence-level estimators target $\nabla J(\theta)$, whereas the token-level estimators target $\nabla\Japp(\theta)$. This difference, together with the different numerical behavior of importance sampling ratios, accounts for the trade-offs discussed below. 

The importance sampling ratio accounts for the main drawback of sequence-level importance sampling. Because $\pi_{\theta}(y\mid x)/\piapp(y\mid x)$ is a product of $|y|$ token-level ratios, even when the per-token discrepancy in the ratios is small, these discrepancies accumulate along the sequence. For long agentic traces, this accumulation can cause sequence-level ratios to span several orders of magnitude. As a result, computing the sequence-level ratio may lead to numerical instability. More importantly, the variance of the corresponding gradient estimator can become enormous. This is because the sample average may be dominated by a small number of samples with large importance sampling ratios, thus reducing the effective sample size of a batch~\citep{kong_note_1992,liu_monte_2001}. 

In comparison, the token-level ratios are better behaved numerically. This is, however, at the cost that the gradient estimator only corresponds to a surrogate $\Japp$ at the current inference policy. At step $k$ of policy optimization, the policy improvement guarantee only ensures that the updated training policy $\pi_{\theta_{k+1}}$ is better than the current inference policy $\piapp_{k}$, i.e., $V_{0}(\pi_{\theta_{k+1}})>V_{0}(\piapp_{k})$. Due to the training--inference mismatch, $\piapp_{k}$ generally differs from the current training policy $\pi_{\theta_{k}}$ and can be worse, i.e., $V_{0}(\piapp_{k})<V_{0}(\pi_{\theta_{k}})$. Consequently, it may occur that $V_{0}(\pi_{\theta_{k+1}})<V_{0}(\pi_{\theta_{k}})$, thus defeating the purpose of the policy update. One apparent remedy is to reduce the mismatch between $\piapp_{k}$ and $\pi_{\theta_{k}}$ and consequently $|V_{0}(\piapp_{k})-V_{0}(\pi_{\theta_{k}})|$ by choosing how the two engines are implemented and how the parameters are synchronized. 

\section{Policy Gradient as a Covariance\label{sec:covariance}}

This section shows that both the policy gradient $\nabla J(\theta)$ and the approximate policy gradient $\nabla\Japp(\theta)$ can be viewed as a \emph{covariance}. This interpretation of the policy gradient yields two main results. First, when the importance sampling ratio is regularized, e.g., by truncation, masking, or clipping (as in Remark~\ref{rem:clipping}), the covariance form recovers invariance to constant shifts in the reward, a key property enjoyed by standard importance sampling. Second, the covariance form admits two distinct unbiased estimators, which correspond to the group-relative methods, including \ac{rloo} and Dr.\ GRPO, in Section~\ref{sec:IS} and \emph{score centering}~\citep{marek_score_2026}, respectively. Comparing the variances of these two estimators provides a guideline for choosing between the two algorithms for estimating the policy gradient. Throughout this section, the prompt $x$ is fixed. As discussed in Remark~\ref{rem:conditional}, the per-prompt results extend to the full objective by drawing $x$ from $\cD$.

\subsection{Unbiased covariance estimation\label{subsec:estimators}}

To prepare for subsequent discussions on sample-average approximation of the policy gradient, this subsection briefly reviews the estimation of a covariance from samples. The results are applied in Sections~\ref{subsec:seq} and~\ref{subsec:token}.

Let $(X_{1},Y_{1}),\dots,(X_{n},Y_{n})$ be independent copies of a pair $(X,Y)$ with $X\in\reals$ and $Y\in\reals^{d}$, matching the setting in which $X$ is a scalar reward and $Y$ a score. Write $\mu_{X}\deq\E[X]$,\ $\mu_{Y}\deq\E[Y]$, and
\[
\bar{X}\deq\frac{1}{n}\sum_{i=1}^{n}X_{i},\qquad
\bar{Y}\deq\frac{1}{n}\sum_{i=1}^{n}Y_{i},\qquad
\tilde{X}\deq X-\mu_{X},\qquad \tilde{Y}\deq Y-\mu_{Y}.
\]
The estimand is the covariance vector $\sigma_{XY}\deq\Cov(X,Y)=\E[\tilde{X}\tilde{Y}]\in\reals^{d}$. Define the moments
\begin{equation}
	\mu_{r0}\deq\E\!\left[\tilde{X}^{\,r}\right],\qquad
	\mu_{r1}\deq\E\!\left[\tilde{X}^{\,r}\tilde{Y}\right]\in\reals^{d},\qquad
	\mu_{r2}\deq\E\!\left[\tilde{X}^{\,r}\cnorm{\tilde{Y}}^{2}\right].
	\label{eq:moments}
\end{equation}
In particular, we have $\mu_{11}=\sigma_{XY}$, $\mu_{20}=\Var(X)$, and $\mu_{02}=\E\cnorm{\tilde{Y}}^{2}=\tr\Cov(Y)$. For an estimator $\hat{\sigma}$ of $\sigma_{XY}$, we write $\Var(\hat{\sigma})\deq\E\norm{\hat{\sigma}-\E\hat{\sigma}}^{2}$, the sum of the variances of individual coordinates, which for an unbiased estimator is its mean squared error.

\begin{defn}[Sample-mean estimators]\label{def:sm}
	For $n\geq2$, define
	\begin{equation}
		\begin{aligned}
			\hat{\sigma}_{1}&\deq\frac{1}{n-1}\sum_{i=1}^{n}\left(X_{i}-\bar{X}\right)\left(Y_{i}-\bar{Y}\right),\\
			\hat{\sigma}_{2}&\deq\frac{1}{n-1}\sum_{i=1}^{n}\left(X_{i}-\bar{X}\right)Y_{i},\\
			\hat{\sigma}_{3}&\deq\frac{1}{n-1}\sum_{i=1}^{n}X_{i}\left(Y_{i}-\bar{Y}\right).
		\end{aligned}
		\label{eq:sm_forms}
	\end{equation}
\end{defn}
The following shows that these three are the same unbiased estimator of $\sigma_{XY}$.

\begin{prop}[Equivalence among sample-mean estimators]\label{prop:onesided}
	$\hat{\sigma}_{1}=\hat{\sigma}_{2}=\hat{\sigma}_{3}$ for every realization of the sample. More generally, for any constants $\alpha\in\reals$ and $\beta\in\reals^{d}$,
	\begin{equation}
		\frac{1}{n-1}\sum_{i=1}^{n}\left(X_{i}-\bar{X}\right)\left(Y_{i}-\beta\right)
		=\frac{1}{n-1}\sum_{i=1}^{n}\left(X_{i}-\alpha\right)\left(Y_{i}-\bar{Y}\right)
		=\hat{\sigma}_{1},
		\label{eq:onesided}
	\end{equation}
	and $\E\,\hat{\sigma}_{1}=\sigma_{XY}$.
\end{prop}

\begin{proof}
	For the first equality in~\eqref{eq:onesided},
	\[
	\sum_{i}\left(X_{i}-\bar{X}\right)\left(Y_{i}-\beta\right)
	=\sum_{i}\left(X_{i}-\bar{X}\right)\left(Y_{i}-\bar{Y}\right)
	+\Bigl(\sum_{i}\left(X_{i}-\bar{X}\right)\Bigr)\left(\bar{Y}-\beta\right),
	\]
	and $\sum_{i}(X_{i}-\bar{X})=0$. The second equality is symmetric, using $\sum_{i}(Y_{i}-\bar{Y})=0$. Setting $\beta=0$ and $\alpha=0$ recovers $\hat{\sigma}_{2}$ and $\hat{\sigma}_{3}$, respectively. For unbiasedness, expand $\sum_{i}(X_{i}-\bar{X})(Y_{i}-\bar{Y})=\sum_{i}X_{i}Y_{i}-n\bar{X}\bar{Y}$ and use $\E[X_{i}Y_{i}]=\sigma_{XY}+\mu_{X}\mu_{Y}$ together with $\E[\bar{X}\bar{Y}]=\Cov(\bar{X},\bar{Y})+\mu_{X}\mu_{Y} =\sigma_{XY}/n+\mu_{X}\mu_{Y}$. The expectation of the sum is thus $(n-1)\sigma_{XY}$.
\end{proof}

Proposition~\ref{prop:onesided} reveals the following two facts. First, it suffices to center either $X$ or $Y$, and performing both has no further effect. Second, knowledge of the population mean $\mu_{Y}$ is of no use when $\bar{X}$ is used for centering because setting $\beta=\mu_{Y}$ in~\eqref{eq:onesided} leads to the same estimator.

In comparison, when $\bar{X}$ is not used, replacing $\bar{Y}$ in $\hat{\sigma}_{3}$ by $\mu_{Y}$ yields a different estimator as defined below.

\begin{defn}[Known-mean estimator]\label{def:km}
	Suppose $\mu_{Y}$ is known. For $n\geq1$, define
	\begin{equation}
		\hat{\sigma}_{\mu_{Y}}\deq\frac{1}{n}\sum_{i=1}^{n}X_{i}\left(Y_{i}-\mu_{Y}\right).
		\label{eq:km0}
	\end{equation}
\end{defn}
The sample-mean estimator and the known-mean estimator differ in how many samples they require. The known-mean estimator is defined at $n=1$, whereas the sample-mean estimator requires $n > 1$. Similarly to what is done in~\eqref{eq:onesided}, we consider a generalization of $\hat{\sigma}_{\mu_{Y}}$ by incorporating a centering constant $b$ for $X$. The constant $b$ is referred to as a baseline since it plays the role of a baseline reward in policy gradient.

\begin{defn}[Known-mean estimator with a baseline]\label{def:kmb}
	For $n\geq1$ and any constant $b\in\reals$, define
	\begin{equation}
		\hat{\sigma}_{\mu_{Y}}(b)\deq\frac{1}{n}\sum_{i=1}^{n}\left(X_{i}-b\right)\left(Y_{i}-\mu_{Y}\right).
		\label{eq:km}
	\end{equation}
\end{defn}

It follows immediately that $\hat{\sigma}_{\mu_{Y}}=\hat{\sigma}_{\mu_{Y}}(0)$. The following shows that this estimator is unbiased for any $b$.

\begin{prop}\label{prop:freeb}
	$\E\,\hat{\sigma}_{\mu_{Y}}(b)=\sigma_{XY}$ for every $n\geq1$ and every constant $b$.
\end{prop}

\begin{proof}
	Writing $X_{i}-b=\tilde{X}_{i}+(\mu_{X}-b)$ and using $\E[\tilde{Y}_{i}]=0$,
	\[
	\E\left[\left(X_{i}-b\right)\tilde{Y}_{i}\right]
	=\E\left[\tilde{X}_{i}\tilde{Y}_{i}\right]
	+\left(\mu_{X}-b\right)\E\left[\tilde{Y}_{i}\right]=\sigma_{XY}. \qedhere
	\]
\end{proof}

Although both $\hat{\sigma}_{1}$ and $\hat{\sigma}_{\mu_{Y}}(b)$ are unbiased, their variances are different.

\begin{prop}[Variances of the estimators]\label{prop:var}
	Define the residual mean $m\deq\mu_{X}-b$ and
	\begin{equation}
		q(m)\deq m^{2}\mu_{02}+2m\,\mu_{12}.
		\label{eq:qdef}
	\end{equation}
	\emph{(i)} For every $n\geq1$,
	\begin{equation}
		\Var\!\left(\hat{\sigma}_{\mu_{Y}}(b)\right)
		=\frac{\mu_{22}-\norm{\mu_{11}}^{2}+q(m)}{n}.
		\label{eq:var_km}
	\end{equation}
	\emph{(ii)} For every $n\geq2$,
	\begin{equation}
		\Var\!\left(\hat{\sigma}_{1}\right)
		=\frac{\mu_{22}-\norm{\mu_{11}}^{2}}{n}
		+\frac{\norm{\mu_{11}}^{2}+\mu_{20}\mu_{02}}{n(n-1)} .
		\label{eq:var_sm}
	\end{equation}
\end{prop}

\begin{proof}
	For~\eqref{eq:var_km}, the summands of~\eqref{eq:km} are independent and identically distributed. This implies
	\begin{equation}
		\Var(\hat{\sigma}_{\mu_{Y}}(b)) = \frac{1}{n}\bigl(\E\cnorm{(X-b)\tilde{Y}}^{2} -\norm{\sigma_{XY}}^{2}\bigr)
		= \frac{1}{n}\bigl(\E(X-b)^{2}\cnorm{\tilde{Y}}^{2} -\norm{\sigma_{XY}}^{2}\bigr).
	\end{equation}
	Expanding $(X-b)^{2}=\tilde{X}^{2}+2m\tilde{X}+m^{2}$ and using~\eqref{eq:moments}, we obtain $\E[(X-b)^{2}\cnorm{\tilde{Y}}^{2}]=\mu_{22}+2m\,\mu_{12}+m^{2}\mu_{02}$. For~\eqref{eq:var_sm}, notice that $\Var(\cdot)$ is a sum over coordinates, and the $k$th entry of $\hat{\sigma}_{1}$ is the ordinary scalar sample covariance of $(X_{i},Y_{ik})$, whose exact variance is $\bigl(\mu_{22}^{(k)}-\sigma_{XY_{k}}^{2}\bigr)/n +\bigl(\sigma_{XY_{k}}^{2}+\mu_{20}\sigma_{Y_{k}}^{2}\bigr)/(n(n-1))$ with $\mu_{22}^{(k)}\deq\E[\tilde{X}^{2}\tilde{Y}_{k}^{2}]$; summing over $k$ gives~\eqref{eq:var_sm}.
\end{proof}

The baseline that minimizes $\Var\!\left(\hat{\sigma}_{\mu_{Y}}(b)\right)$ can be obtained by minimizing $q$ in~\eqref{eq:var_km}. The minimum of $q$ is attained at 
\begin{equation}
	m^{\star}\deq-\frac{\mu_{12}}{\mu_{02}}	
\end{equation}
with $q(m^{\star})=-\mu_{12}^{2}/\mu_{02}\leq0$. Thus, the minimum-variance baseline is given by
\begin{equation}
	b^{\star}=\mu_{X}-m^{\star}=\mu_{X}+\frac{\mu_{12}}{\mu_{02}}
	=\frac{\E\left[X\norm{Y-\mu_{Y}}^{2}\right]}{\E\norm{Y-\mu_{Y}}^{2}} .
	\label{eq:bstar}
\end{equation}

Proposition~\ref{prop:var} shows that the choice of $b$ determines which of the two unbiased estimators has the smaller variance. When the baseline is set to the population mean of $X$, i.e., $b=\mu_{X}$ or equivalently $m=0$, the variance of the known-mean estimator is smaller by $\bigl(\norm{\mu_{11}}^{2}+\mu_{20}\mu_{02}\bigr)/\bigl(n(n-1)\bigr)$, which is of order $O(n^{-2})$. When the minimum-variance baseline $b^{\star}$ is used, the variance is further reduced by $\mu_{12}^{2}/(n\mu_{02})$, which is of order $O(n^{-1})$. More generally, $\hat{\sigma}_{\mu_{Y}}(b)$ has strictly smaller variance than $\hat{\sigma}_{1}$ if and only if
\begin{equation}
	\abs{m-m^{\star}}<\rho_{n}
	\deq\sqrt{\frac{\norm{\mu_{11}}^{2}+\mu_{20}\mu_{02}}{(n-1)\mu_{02}}
		+\frac{\mu_{12}^{2}}{\mu_{02}^{2}}}.
	\label{eq:band}
\end{equation}
Conversely, when $m$ is far from $m^{\star}$, the variance of $\hat{\sigma}_{\mu_{Y}}(b)$ can be much larger than that of $\hat{\sigma}_{1}$, and the gap grows without bound as $\abs{m}$ grows.

\subsection{The sequence-level estimand\label{subsec:estimand}}

Denote by
\begin{equation}
	\score(y\mid x)\deq\nabla_{\theta}\log\pi_{\theta}(y\mid x)
	\label{eq:score_def}
\end{equation}
the score. Recall that $\E_{y\sim\pi_{\theta}}[\score(y\mid x)] =\sum_{y}\nabla_{\theta}\pi_{\theta}(y\mid x)=\nabla_{\theta}1=0$. This implies that the policy gradient $\nabla J(\theta)$ in~\eqref{eq:pg_seq} can be viewed as a covariance:
\begin{equation}
	\nabla J(\theta)
	=\E_{y\sim\pi_{\theta}}\left[R(x,y)\,\score(y\mid x)\right]
	=\Cov_{\pi_{\theta}}\!\left(R(x,y),\,\score(y\mid x)\right).
	\label{eq:pgcov}
\end{equation}
Denote the (sequence-level) importance sampling ratio by
\begin{equation}
	w(y\mid x)\deq\frac{\pi_{\theta}(y\mid x)}{\piapp(y\mid x)}.
	\label{eq:w_def}
\end{equation}
As discussed in Section~\ref{subsec:comparison}, $w$ is a product of $\abs{y}$ token-level ratios, and the resulting sample average can be dominated by a few rollouts. The standard method is to regularize the importance sampling ratio by replacing $w$ with $f(w)$ for some $f\colon\reals_{+}\to\reals_{+}$ satisfying $f(1) = 1$. Common choices are
\begin{equation}
	\begin{aligned}
		f(w)&=w                                 && \text{(standard importance sampling)},\\
		f(w)&=\min(w,C)                         && \text{(truncated importance sampling)},\\
		f(w)&=w\,\mathbf{1}\!\left[w\in[l,u]\right] && \text{(masked importance sampling)},\\
		f(w)&=\mathrm{clip}(w,1-\epsilon,1+\epsilon) && \text{(ratio clipping)}.
	\end{aligned}
	\label{eq:f_choices}
\end{equation}
The property that $f(1) = 1$ ensures that no regularization is imposed when the two policies agree. 

It is convenient to absorb the weight into the score by defining the \emph{weighted score}
\begin{equation}
	\wscore(y\mid x)\deq f\bigl(w(y\mid x)\bigr)\,\score(y\mid x).
	\label{eq:wscore_def}
\end{equation}
When importance sampling is used with regularization, the quantity in~\eqref{eq:pg_seq_IS} can be written as $\E_{\piapp}[R\,\wscore]$. We take the covariance
\begin{equation}
	\sigma^{\mathrm{seq}}\deq\Cov_{\piapp}\!\left(R(x,y),\,\wscore(y\mid x)\right)
	=\E_{\piapp}\left[R\,\wscore\right]-\E_{\piapp}[R]\cdot\E_{\piapp}\left[\wscore\right],
	\label{eq:estimand}
\end{equation}
as the sequence-level estimand instead of $\E_{\piapp}[R\,\wscore]$ for the following reasons. First, \eqref{eq:estimand} reduces to $\nabla J(\theta)$ if either $f$ is the identity mapping $\mathrm{id}$ (i.e., when standard importance sampling is used) or $\piapp=\pi_{\theta}$: In the former case,
\begin{equation}
	\E_{\piapp}\left[\wscore\right]
	=\sum_{y}\piapp(y\mid x)\frac{\pi_{\theta}(y\mid x)}{\piapp(y\mid x)}\score(y\mid x)
	=\E_{\pi_{\theta}}\left[\score\right]=0.
	\label{eq:wscore_zero}
\end{equation}
In the latter case, $w\equiv1$, and thus $\wscore=f(1)\score=\score$. Second, $\sigma^{\mathrm{seq}}$ is unchanged if the reward is shifted by a constant, $R\mapsto R+c$, whereas $\E_{\piapp}[R\,\wscore]$ is not.

\subsection{Sequence-level covariance estimation\label{subsec:seq}}

Since the sequence-level estimand $\sigma^{\mathrm{seq}}$ in~\eqref{eq:estimand} is a covariance, the results of Section~\ref{subsec:estimators} apply directly. Fix a prompt $x$, draw $y_{1},\dots,y_{G}\sim\piapp(\cdot\mid x)$ independently, and set $X_{i}=R(x,y_{i})$ and $Y_{i}=\wscore(y_{i}\mid x)$ with $n=G$. The following proposition shows that the sample-mean estimator of $\sigma^{\mathrm{seq}}$ yields the \ac{rloo} algorithm and, up to a multiplicative factor, the Dr.\ GRPO algorithm.

\begin{prop}[\ac{rloo} and Dr.\ GRPO are the sample-mean estimator]\label{prop:rloo}
	Let $b_{i}$ be the leave-one-out baseline of~\eqref{eq:loo_baseline}. Then
	\begin{equation}
		\frac{1}{G}\sum_{i=1}^{G}\left(X_{i}-b_{i}\right)Y_{i}
		=\hat{\sigma}_{1},
		\qquad
		\frac{1}{G}\sum_{i=1}^{G}\left(X_{i}-\bar{R}\right)Y_{i}
		=\frac{G-1}{G}\,\hat{\sigma}_{1},
		\label{eq:rloo_is_cov}
	\end{equation}
	where $\bar{R}$ is the group mean~\eqref{eq:group_mean}.
\end{prop}

\begin{proof}
	By~\eqref{eq:loo_identity}, $X_{i}-b_{i}=\frac{G}{G-1}(X_{i}-\bar{R})$ for each $i$. Thus the first sum in~\eqref{eq:rloo_is_cov} becomes $\frac{1}{G-1}\sum_{i}(X_{i}-\bar{R})Y_{i}=\hat{\sigma}_{2}$, which is $\hat{\sigma}_{1}$ by Proposition~\ref{prop:onesided}. The second identity is immediate from the same expression.
\end{proof}

With standard importance sampling, i.e., $f = \mathrm{id}$, we have $\mu_{Y} \deq \E_{\piapp}[\wscore] = 0$; hence score centering is implicit. When $f\neq\mathrm{id}$, we have $\mu_{Y}\neq0$ in general, and one may ask whether score centering provides any benefit. Exact computation of $\mu_{Y}$ is intractable since it involves a sum over all possible responses. More importantly, the following proposition shows that score centering with $\mu_{Y}$ is unnecessary when the sample-mean reward is used as the baseline, even when $f\neq\mathrm{id}$.

\begin{prop}[Sequence-level centering with the group-mean baseline]\label{prop:noop}
	With the group-mean baseline,
	\begin{equation}
		\frac{1}{G-1}\sum_{i=1}^{G}\left(X_{i}-\bar{R}\right)\left(Y_{i}-\mu_{Y}\right)
		=\frac{1}{G-1}\sum_{i=1}^{G}\left(X_{i}-\bar{R}\right)Y_{i} = \hat{\sigma}_1.
		\label{eq:noop}
	\end{equation}
\end{prop}

\begin{proof}
	The result immediately follows from Proposition~\ref{prop:onesided} by setting $\beta = \mu_{Y}$.
\end{proof}

Since both \ac{rloo} and Dr.\ GRPO can be viewed as the sample-mean estimator $\hat{\sigma}_1$ up to a multiplicative factor, Proposition~\ref{prop:noop} implies that score centering is unnecessary for both algorithms.

\paragraph{Comparison with the variance-optimal baseline}
Suppose that standard importance sampling is used. Then $\mu_{Y} = 0$. From~\eqref{eq:bstar}, the optimal baseline is given by $b^{\star}=\E[X\norm{Y}^{2}]/\E\norm{Y}^{2}$. This is the variance-optimal reward baseline of \citet[Theorem~11]{greensmith_variance_2004}. The mean reward, which is the baseline used by the group-relative methods, is optimal only when $\mu_{12}=0$, as also shown by \citet[Corollary~12]{greensmith_variance_2004}. In language model post-training, $\mu_{12}=\Cov(R,\norm{\score}^{2})$ is not negligible, since $\norm{\score}^{2}$ grows with response length, and length is empirically correlated with reward. This indicates that using the mean reward as the baseline may be far from optimal.

\subsection{The token-level estimand\label{subsec:tok_estimand}}

At the token level, the approximate policy gradient $\nabla\Japp(\theta)$ is computed, and the importance sampling ratio is applied per token position as in~\eqref{eq:trpo_grad_token} and~\eqref{eq:trpo_grad_token_reward}. Write $h_{t}\deq(x,y_{<t})$ for the prefix,
\begin{equation}
	w(y_t \mid h_t)\deq\frac{\pi_{\theta}(y_{t}\mid h_{t})}{\piapp(y_{t}\mid h_{t})},
	\qquad
	\wscore_{t}\deq f(w(y_t \mid h_t))\,\nabla_{\theta}\log\pi_{\theta}(y_{t}\mid h_{t}),
	\label{eq:tok_defs}
\end{equation}
and let
\begin{equation}
	\bar{s}(h_t)
	\deq\E_{\piapp}\left[\wscore_{t} \mid y_{<t}\right]
	=\sum_{v\in\cV}\piapp(v\mid h_t)\,f(w(v\mid h_t))\,\nabla_{\theta}\log\pi_{\theta}(v\mid h_t)
	\label{eq:sbar_h}
\end{equation}
be the mean weighted score at a given prefix $h_t$, where $\cV$ denotes the token vocabulary, i.e., the action space $\cA$ in Section~\ref{subsec:MDP}. Because the expectation in~\eqref{eq:sbar_h} is a finite sum over $\cV$, computing $\bar{s}(h)$ is much more tractable than computing the mean score $\mu_{Y}$ at the sequence level in Section~\ref{subsec:seq}. As in the sequence-level case, $\bar{s}(h) = 0$ for any $h$ when $f = \mathrm{id}$. We assume that the sampling policy terminates in finite expected time and that the weighted score has a bounded conditional mean:
\begin{equation}
	\E_{\piapp}\left[\abs{y}\right]<\infty,
	\qquad
	\sup_{h}\ \E_{v\sim\piapp(\cdot\mid h)}\norm{f(w(v\mid h))\,\nabla_{\theta}\log\pi_{\theta}(v\mid h)}<\infty .
	\label{eq:regularity}
\end{equation}
To incorporate regularized importance sampling, we generalize $\nabla\Japp(\theta)$ in~\eqref{eq:trpo_grad_token} and define the token-level estimand
\begin{equation}
	\sigma^{\mathrm{tok}}
	\deq\E_{\piapp}\left[\sum_{t=0}^{|y|-1}f(w(y_{t}\mid h_{t}))\,A^{\piapp}(h_{t},y_{t})\,\nabla_{\theta}\log\pi_{\theta}(y_{t}\mid h_{t})\right].
	\label{eq:tok_estimand_adv}
\end{equation}
Write $Q_{t}\deq\E_{\piapp}[R\mid y_{\leq t}]$ for the state-action value and $V(h_{t})\deq\E_{\piapp}[R\mid y_{<t}]$ for the state value under $\piapp$. The following lemma shows that $\sigma^{\mathrm{tok}}$ can be expressed as a sum of conditional covariances.
\begin{lem} The token-level estimand
\begin{equation}
	\sigma^{\mathrm{tok}}
	=\E_{\piapp}\left[\sum_{t=0}^{\abs{y}-1}\Cov_{\piapp}\left(\wscore_{t},Q_{t}\;\big|\;y_{<t}\right)\right].
	\label{eq:tok_estimand}
\end{equation}
\end{lem}

\begin{proof}
For brevity, we assume that all the expectations and covariances are taken under $\piapp$ and drop $\piapp$ from all the subscripts. Using the facts $\E[Q_{t}\mid y_{<t}]=V(h_{t})$ and $A^{\piapp}(h_{t}, y_{t}) = Q_t - V(h_{t})$, it can be seen that
\begin{equation}
\begin{aligned}
	&\Cov\bigl(\wscore_{t},Q_{t}\;\big|\;y_{<t}\bigr) = \E\left[\wscore_{t} (Q_{t} - \E[Q_{t} \mid y_{<t}])\;\big|\;y_{<t}\right]\\ 
	&\qquad = \E\left[\wscore_{t} (Q_{t} - V(h_{t}))\;\big|\;y_{<t}\right]
	= \E\left[\wscore_{t} A^{\piapp}(h_{t}, y_{t})\;\big|\;y_{<t}\right].
\end{aligned}
\label{eq:cov_token_Q_centering}
\end{equation}
Thus, 
\begin{equation}
\begin{aligned}
	& \E\left[\sum_{t=0}^{\abs{y}-1}\Cov\left(\wscore_{t},Q_{t}\;\big|\;y_{<t}\right)\right]=\E\left[\sum_{t=0}^{\infty}\mathbf{1}\bigl[t<\abs{y}\bigr]\E\left[\wscore_{t}A^{\piapp}(h_{t},y_{t})\;\big|\;y_{<t}\right]\right]\\
	& \qquad=\sum_{t=0}^{\infty}\E\left[\mathbf{1}\bigl[t<\abs{y}\bigr]\E\left[\wscore_{t}A^{\piapp}(h_{t},y_{t})\;\big|\;y_{<t}\right]\right]=\sum_{t=0}^{\infty}\E\left[\E\left[\mathbf{1}\bigl[t<\abs{y}\bigr]\wscore_{t}A^{\piapp}(h_{t},y_{t})\;\big|\;y_{<t}\right]\right]\\
	& \qquad=\sum_{t=0}^{\infty}\E\left[\mathbf{1}\bigl[t<\abs{y}\bigr]\wscore_{t}A^{\piapp}(h_{t},y_{t})\right]=\E\left[\sum_{t=0}^{\abs{y}-1}\wscore_{t}A^{\piapp}(h_{t},y_{t})\right]=\sigma^{\mathrm{tok}}.
\end{aligned}
\end{equation}
The interchange of the summation and the expectation is permitted under the conditions in~\eqref{eq:regularity}. The third equality holds because the event $\{\abs{y}>t\}$ is a function of $y_{<t}$, as it is determined by whether $\eos$ occurs among $y_{0},\dots,y_{t-1}$.
\end{proof}

The form in~\eqref{eq:tok_estimand} does not depend on the score being centered since a covariance only requires one of its arguments to be centered, as achieved by $V(h_{t}) = \E_{\piapp}[Q_{t}\mid y_{<t}]$ in~\eqref{eq:cov_token_Q_centering}. Alternatively, one may center the score $\wscore_{t}$ by expressing
\begin{equation}
\Cov_{\piapp}\bigl(\wscore_{t},Q_{t}\;\big|\;y_{<t}\bigr) = \E_{\piapp}\bigl[\bigl(\wscore_{t}-\E_{\piapp}\left[\wscore_{t} \mid y_{<t}\right]\bigr)\,Q_{t}\bigm|y_{<t}\bigr] = \E_{\piapp}\bigl[\bigl(\wscore_{t}-\bar{s}(h_{t})\bigr)\,Q_{t}\bigm|y_{<t}\bigr].
\label{eq:cov_token_score_centering} 
\end{equation}
As will be shown in Proposition~\ref{prop:token_estimand}, score centering leads to an equivalent and more useful form. Define
\begin{equation}
	D\deq\sum_{t=0}^{\abs{y}-1}\left(\wscore_{t}-\bar{s}(h_{t})\right).
	\label{eq:Ddef}
\end{equation}	

\begin{lem}\label{lem:mds} $\E_{\piapp}[D]=0$.
\end{lem}

\begin{proof}
	Let $M_{t}\deq\mathbf{1}\bigl[t<\abs{y}\bigr]\bigl(\wscore_{t}-\bar{s}(h_{t})\bigr)$. Both the prefix $h_{t}=(x,y_{<t})$ and the event $\{\abs{y}>t\}$ are functions of $y_{<t}$. By~\eqref{eq:sbar_h},
	\[
	\E\left[M_{t}\mid y_{<t}\right]
	=\mathbf{1}\bigl[t<\abs{y}\bigr]\E\left[\wscore_{t} - \bar{s}(h_{t}) \mid y_{<t}\right]
	=\mathbf{1}\bigl[t<\abs{y}\bigr]
	\left(\E\left[\wscore_{t} \mid y_{<t}\right]-\bar{s}(h_{t})\right)=0.
	\]
	This shows that $\{M_{t}\}$ is a martingale difference sequence. Therefore, the partial sums of $\{M_{t}\}$ form a martingale, and $\abs{y}$ is a stopping time. The conditions in~\eqref{eq:regularity} bound $\E[\norm{M_{t}}\mid y_{<t}]$ uniformly, so optional stopping applies and $\E[D]=0$.
\end{proof}

The following proposition shows that $\sigma^{\mathrm{tok}}$ can be expressed as a single covariance.

\begin{prop}\label{prop:token_estimand}
	 The token-level estimand $\sigma^{\mathrm{tok}} = \Cov_{\piapp}(R, D)$.
\end{prop}

\begin{proof}
Expanding the conditional covariance in~\eqref{eq:tok_estimand} with score centering in~\eqref{eq:cov_token_score_centering} yields
\begin{equation}
	\begin{aligned}
		& \E\left[\sum_{t=0}^{\abs{y}-1}\E\bigl[\bigl(\wscore_{t}-\bar{s}(h_{t})\bigr)\,Q_{t}\bigm|y_{<t}\bigr]\right]=\E\left[\sum_{t=0}^{\infty}\mathbf{1}\bigl[t<\abs{y}\bigr]\,\E\bigl[\bigl(\wscore_{t}-\bar{s}(h_{t})\bigr)\,Q_{t}\bigm|y_{<t}\bigr]\right]\\
		& \qquad=\sum_{t=0}^{\infty}\E\left[\mathbf{1}\bigl[t<\abs{y}\bigr]\,\E\bigl[\bigl(\wscore_{t}-\bar{s}(h_{t})\bigr)\,Q_{t}\bigm|y_{<t}\bigr]\right]=\sum_{t=0}^{\infty}\E\left[\mathbf{1}\bigl[t<\abs{y}\bigr]\,\E\bigl[\bigl(\wscore_{t}-\bar{s}(h_{t})\bigr)\,\E[R\mid y_{\leq t}]\bigm|y_{<t}\bigr]\right]\\
		& \qquad=\sum_{t=0}^{\infty}\E\left[\E\bigl[\E\bigl[\mathbf{1}\bigl[t<\abs{y}\bigr]\,\bigl(\wscore_{t}-\bar{s}(h_{t})\bigr)\,R\mid y_{\leq t}\bigr]\bigm|y_{<t}\bigr]\right]=\sum_{t=0}^{\infty}\E\left[\mathbf{1}\bigl[t<\abs{y}\bigr]\,\bigl(\wscore_{t}-\bar{s}(h_{t})\bigr)\,R\right]\\
		& \qquad=\E\left[\sum_{t=0}^{\infty}\mathbf{1}\bigl[t<\abs{y}\bigr]\bigl(\wscore_{t}-\bar{s}(h_{t})\bigr)R\right]=\E[R\,D].
	\end{aligned}
\end{equation}
The interchange of the summation and the expectation is permitted under the conditions in~\eqref{eq:regularity}. The fourth equality uses the fact that both $\wscore_{t}-\bar{s}(h_{t})$ and the event $\{\abs{y}>t\}$ are functions of $y_{\leq t}$.
\end{proof}

\subsection{Token-level covariance estimation\label{subsec:token}}

Since the token-level estimand in Proposition~\ref{prop:token_estimand} is a single covariance, the results of Section~\ref{subsec:estimators} immediately apply with $X=R$, $Y=D$, and $n=G$. In addition, Lemma~\ref{lem:mds} implies $\mu_{Y} = \E[D] = 0$. Suppose that $y_{1},\dots,y_{G}\sim\piapp(\cdot\mid x)$ are drawn independently. The corresponding pairs $(R_{i},D_{i})$, $i=1,\dots,G$, are independent and identically distributed. The estimator we mainly consider is the known-mean estimator in~\eqref{eq:km} with $\mu_{Y}=0$, namely
\[
\hat{\sigma}_{\mu_{Y}}(b)=\frac{1}{G}\sum_{i=1}^{G}(R_{i}-b)D_{i},
\]
 	which is unbiased for every constant $b$ and every $G\geq1$ by Proposition~\ref{prop:freeb}. 
 	
In comparison, the sample-mean estimator analogous to $\hat{\sigma}_2$ in Definition~\ref{def:sm} is given by
\[
\frac{1}{G-1}\sum_{i=1}^{G}(R_{i}-\bar{R})D_{i}.
\]
It also requires score centering because it involves $D_{i}$. Recall that, by Proposition~\ref{prop:noop}, score centering for the sample-mean estimator in~\eqref{eq:rloo_is_cov} is optional at the sequence level regardless of $f$, as long as the group-mean baseline $b=\bar{R}$ is used. In contrast, the following proposition shows that score centering is no longer optional at the token level even when $b=\bar{R}$; removing score centering would generally yield a different estimator. 

\begin{prop}[Token-level centering with the group-mean baseline]\label{prop:tokcenter}
	With the group-mean baseline, the uncentered estimator
	\begin{equation}
		\frac{1}{G-1}\sum_{i=1}^{G}\bigl(R_{i}-\bar{R}\bigr)\sum_{t=0}^{\abs{y_i}-1}\wscore_{i,t},
		\label{eq:tokuncentered}
	\end{equation}
	obtained by applying $f$ to~\eqref{eq:token_grpo}, is unbiased for
	\begin{equation}
		\sigma^{\mathrm{tok}}+\Cov_{\piapp}\Bigl(R,\ \sum_{t=0}^{\abs{y}-1}\bar{s}(h_{t})\Bigr).
		\label{eq:tokbias}
	\end{equation}
\end{prop}
\begin{proof}
	Split the inner sum as $\sum_{t}\wscore_{i,t}=D_{i}+\sum_{t}\bar{s}(h_{i,t})$ and rewrite~\eqref{eq:tokuncentered} as
	\[
	\frac{1}{G-1}\sum_{i=1}^{G}\bigl(R_{i}-\bar{R}\bigr) D_{i} + \frac{1}{G-1}\sum_{i=1}^{G}\bigl(R_{i}-\bar{R}\bigr)\sum_{t}\bar{s}(h_{i,t}).
	\] 
	Each of the two terms is an ordinary sample covariance with one-sided centering on $R$. These are unbiased for $\Cov(R,D)=\sigma^{\mathrm{tok}}$ and for $\Cov(R,\sum_{t}\bar{s}(h_{t}))$, respectively.
\end{proof}

\begin{rem}
	The discrepancy vanishes when $f=\mathrm{id}$ since $\bar{s}(h) = 0$ for any $h$, and is otherwise nonzero in general. The sequence-level cancellation fails here because the token-level mean score $\sum_{t}\bar{s}(h_{i,t})$ varies across rollouts, whereas the sequence-level mean score $\mu_{Y}$ in~\eqref{eq:noop} is a single constant and can be annihilated by $\sum_{i}(R_{i}-\bar{R})$. 
\end{rem}
	
The optimal baseline can be obtained as
\begin{equation}
	b^{\star}=\frac{\E\bigl[R\norm{D}^{2}\bigr]}{\E\norm{D}^{2}}
	\label{eq:tokbstar}
\end{equation}	
following~\eqref{eq:bstar}. Recall $b^{\star} - \mu_{X} = -m^{\star} = \mu_{12}/\mu_{02}$. Since $\mu_{02}>0$, the sign of $\mu_{12}$ determines how $b^{\star}$ compares with the mean reward $\mu_{X}$. In the token-level setting, we have $\mu_{12} = \E[(R - \E[R]) \norm{D}^2]$. Since $\norm{D}^{2}$ grows with response length, and a longer response empirically leads to higher reward, we expect $\mu_{12} > 0$ to often hold in practice. This implies that it is preferable to set the baseline above the mean reward.

One should keep in mind that the expectations in~\eqref{eq:tokbstar} are taken over rollouts from a \emph{single} prompt. Thus, $b^{\star}$ is prompt-dependent, and no constant baseline remains optimal across multiple prompts in general. Estimating $b^{\star}$ per prompt instead would still require multiple rollouts for the current prompt.
\section{Summary}

We have examined the roles of baselines and score centering in policy gradient estimation when responses are generated off-policy and the importance sampling ratio is regularized. Without regularization, sequence-level and token-level importance sampling yield unbiased gradient estimators for different objectives: the former for the true objective $J$ and the latter for the surrogate $\Japp$. At either level, regularizing the ratio breaks the invariance of the importance-weighted gradient to constant shifts in the reward. Choosing the covariance between the reward and the weighted score as the estimand recovers this invariance, and the covariance reduces to the policy gradient when the ratio is not regularized. At the sequence level, the estimand is a single covariance by definition; at the token level, it is an expected sum of conditional covariances but can also be written as a single covariance (Proposition~\ref{prop:token_estimand}). Hence, the two classes of covariance estimators in Section~\ref{subsec:estimators} apply at both levels.

The covariance perspective clarifies the roles of baselines and score centering. Both perform centering in covariance estimation, but they differ in the choice of the mean used for centering. The group-mean baseline uses the sample mean and yields the sample-mean estimator, which coincides with \ac{rloo} and, up to a multiplicative factor, with Dr.\ GRPO (Proposition~\ref{prop:rloo}). With a sample mean, centering the reward, centering the score, and centering both are all equivalent (Proposition~\ref{prop:onesided}). Score centering uses the population mean instead and yields the known-mean estimator, which is unbiased for any constant baseline (Proposition~\ref{prop:freeb}) and is defined even for a single sample. The two estimators have different variances (Proposition~\ref{prop:var}). The known-mean estimator has strictly smaller variance if and only if its baseline is sufficiently close to the variance-optimal baseline $b^{\star}$, and it can become arbitrarily worse as the baseline moves away from $b^{\star}$.

Whether both forms of centering should be used simultaneously depends on whether importance sampling is applied at the sequence level or at the token level. At the sequence level, score centering is redundant once the group-mean baseline is used, regardless of whether the ratio is regularized (Proposition~\ref{prop:noop}). Consequently, the population mean of the sequence-level weighted score, which is intractable to compute, is not needed. At the token level, while score centering is far more computationally tractable, it is no longer optional. Without score centering, the group-mean estimator is biased for the token-level estimand; its expectation contains an extra term (Proposition~\ref{prop:tokcenter}) that vanishes under standard importance sampling but is nonzero in general when the ratio is regularized.

These results suggest the following practice. With sequence-level importance sampling, the group-mean baseline alone suffices, and score centering should be omitted. With token-level importance sampling and a regularized ratio, score centering should be applied, either together with the group-mean baseline or with a constant baseline; the latter has smaller variance only if the baseline is sufficiently close to $b^{\star}$, which is prompt-dependent and must be estimated from multiple rollouts. Finally, because both families of estimators degrade as the training--inference mismatch grows, it remains useful to reduce the mismatch at its source, e.g., through matched kernels, batch-invariant computation, and consistent arithmetic precision across the training and inference engines.

\bibliographystyle{apalike}
\bibliography{refs-sh}

\end{document}